\documentclass{article}

\usepackage[preprint]{neurips_2019}

\usepackage[utf8]{inputenc} 
\usepackage[T1]{fontenc}    
\usepackage{hyperref}       
\usepackage{url}            
\usepackage{booktabs}       
\usepackage{amsfonts}       
\usepackage{nicefrac}       
\usepackage{microtype}      

\usepackage{amsmath}
\usepackage{amssymb}
\usepackage{amsthm}
\usepackage{caption}
\usepackage{subcaption}
\usepackage{graphicx}
\usepackage[dvipsnames]{xcolor}
\usepackage{xfrac}
\usepackage{enumitem}
\usepackage{wrapfig}

\usepackage[title]{appendix}
\usepackage{etoolbox}
\BeforeBeginEnvironment{appendices}{\clearpage}

\newcommand{\bx}{\boldsymbol{x}}

\newcommand{\bdelta}{\boldsymbol{\delta}}
\newcommand{\by}{\boldsymbol{y}}
\newcommand{\bc}{\boldsymbol{c}}
\newcommand{\bb}{\boldsymbol{b}}
\newcommand{\bh}{\boldsymbol{h}}
\newcommand{\be}{\boldsymbol{e}}
\newcommand{\bu}{\boldsymbol{u}}
\newcommand{\bz}{\boldsymbol{z}}
\newcommand{\bp}{\boldsymbol{p}}

\newcommand{\bW}{\boldsymbol{W}}
\newcommand{\bL}{\boldsymbol{L}}

\newcommand{\bX}{\boldsymbol{X}}
\newcommand{\bE}{\boldsymbol{E}}

\newcommand{\bA}{\boldsymbol{A}}

\renewcommand{\b}[1]{\textbf{#1}}

\newcommand{\cX}{\mathcal{X}}
\newcommand{\cY}{\mathcal{Y}}

\newcommand{\f}{\mathtt{f}}
\newcommand{\g}{\mathtt{g}}
\renewcommand{\v}{\mathtt{v}}
\newcommand{\h}{\mathtt{h}}

\newcommand{\p}{\mathtt{phist}}

\newcommand{\bbR}{\mathbb{R}}

\renewcommand{\vec}[1]{\text{vec}\left(#1\right)}
\newcommand{\nullspace}[1]{\text{null}\left(#1\right)}
\newcommand*{\defeq}{\stackrel{\text{def}}{=}}
\newcommand{\indicate}[1]{{1}{\{#1\}}}

\newcommand{\hide}[1]{}

\title{Discriminative structural graph classification}

\author{%
  Younjoo Seo \\
  EPFL\\
  \texttt{youngjoo.seo@epfl.ch} \\
  \And
   Andreas Loukas \\
   EPFL \\
   \texttt{andreas.loukas@epfl.ch}
   \And
   Nathana\"el Perraudin \\
   SDSC \\
   \texttt{nathanael.perraudin@sdsc.ethz.ch} \\
}

\PassOptionsToPackage{numbers, compress}{natbib}

\begin{document}

\renewcommand{\paragraph}[1]{\vspace{0mm}\noindent\textbf{#1}}
\newcommand{\Section}[1]{\vspace{-2mm}\section{#1}\vspace{-2mm}}
\newcommand{\Subsection}[1]{\vspace{-1mm}\subsection{#1}\vspace{-1mm}}

\newtheoremstyle{exampstyle}
  {1\topsep} 
  {0.2\topsep} 
  {} 
  {} 
  {\bfseries} 
  {.} 
  {.5em} 
  {} 

\newtheorem{theorem}{Theorem}[section]
\newtheorem{proposition}{Proposition}[section]
\newtheorem{corollary}{Corollary}[section]
\newtheorem{lemma}{Lemma}[section]
\newtheorem{claim}{Claim}[section]
\newtheorem{definition}{Definition}[section]

\maketitle

\begin{abstract}
This paper focuses on the discrimination capacity of aggregation functions: these are the permutation invariant functions used by graph neural networks to combine the features of nodes. Realizing that the most powerful aggregation functions suffer from a dimensionality curse, we consider a restricted setting. In particular, we show that the standard sum and a novel histogram-based function have the capacity to discriminate between any fixed number of inputs chosen by an adversary. Based on our insights, we design a graph neural network aiming, not to maximize discrimination capacity, but to learn discriminative graph representations that generalize well. Our empirical evaluation provides evidence that our choices can yield benefits to the problem of structural graph classification.
\end{abstract}

\Section{Introduction}

Neural networks tailored to the graph classification problem commonly rely on two permutation-invariant primitives (e.g., ~\citep{4700287,gilmer2017neural,hamilton2017inductive}): a) a local aggregation function that updates each node's representation by combining the features in its vicinity, and b) a global aggregation (also called pooling or readout) function that fuses all learned node embeddings to obtain a graph representation.

To attain the maximum discrimination power possible, both local and global aggregation functions should be injective. 
Motivated by this observation and inspired by work on deep sets~\citep{NIPS2017_6931}, it was recently shown that simple summation suffices to attain injectivity if the features have first been passed through a (possibly involved) function acting independently on each feature vector~\citep{xu2018powerful}.
This can serve as a theoretical justification for using sum aggregators preceded by a multi-layer perceptron (MLP)~\citep{4700287,duvenaud2015convolutional,xu2018powerful}. In addition, an equivalence was established between the capacity of graph neural networks (GNN) with injective aggregative functions and that of the  Weisfeiler-Lehman (WL) test for graph isomorphism~\citep{weisfeiler1968reduction,weisfeiler2006construction}.

Differently, we argue that maximizing the discrimination capacity of aggregation functions (and GNN) could be both unfeasible and undesirable in the context of learning. 
In fact, our first contribution is to derive a curse of dimensionality lower bound for aggregation functions confirming that 
even approximate injectivity necessitates an {exponential} dependence on the ambient space dimension. 
This implies that, as usual, there exists a trade-off between the capacity of learners to distinguish between distinct objects in the training set and their ability to generalize over unseen examples. Thus, as an overly simple aggregation will lack discrimination power, a network using injective aggregation functions will likely overfit the training set. 

Aiming for neither ends of the trade-off, we advocate for aggregation functions with \textit{bounded} discrimination capacity, i.e., that are capable of distinguishing between a fixed cardinality set of inputs. However, to ensure that there are no blindspots, this should hold for \textit{any} such set of inputs---even if the latter were chosen by an adversary.  
Our second contribution is to derive such bounded injectivity guarantees for the standard sum and a newly introduced histogram-based aggregation function, for the case when they are preceded by a universal approximator that acts independently on the features of each node. Though leading to fewer parameters, using a histogram is shown to be as discriminative as the sum.

For our third and final contribution, we design a GNN aiming not to maximize discrimination capacity, but to learn discriminative representations that generalize well. Taking a purist perspective, we consider the challenging structural classification setting, where one needs to learn graph representations without relying on node or edge attributes. 
Our GNN is built by hierarchically composing two jointly-optimized structural embedding networks: the first computes node representations, whereas the second combines the representations of a random subset of nodes to output a graph representation. Sampling helps us both to reduce the computation time as well as to combat overfitting. Moreover, employing histograms helps reduce the number of trainable parameters.

Our experiments provide evidence that our choices lead to structural embeddings that generalize well: our network attains competitive accuracy in structural graph classification over 9 benchmark datasets, often outperforming the 14 baselines (6 graph kernels and 8 GNNs) we compare it to. 
Interestingly, our network also frequently outperforms WL graph kernels~\citep{shervashidze2011weisfeiler}, reinforcing our claim that, in practice, the achievable capacity of graph neural networks (and graph kernels) can significantly deviate from theoretical limits constructed based on injectivity assumptions.

\Subsection{Our theoretical results in view of the literature}

A number of recent papers have studied theoretically permutation invariance and equivariance in the context of GNN: 
The first group of works focuses on the space of functions that act on graphs. The work by \cite{DBLP:journals/corr/abs-1812-09902} brought forth a characterization all invariant and equivariant linear layers and intriguingly demonstrated that their dimension is 2 and 15, respectively.  Moreover, \citet{maron2019universality} and \citet{keriven2019universal} derived a universal approximation theorem for (a specific class of one hidden layer) invariant and equivariant networks, respectively. 

The second group of works
instead of considering the space of all functions, 
study those functions that can be implemented by iteratively aggregating the values of neighbors of each node (or of all nodes in the pooling stage). Extending the work of~\citet{NIPS2017_6931}, \citet{xu2018powerful} demonstrated that the popular sum aggregation amounts to a universal approximator. Moreover, the same work provided a bound on the capacity of GNN w.r.t. graph classification proving that, when injective aggregation functions are employed, GNN are as discriminative as the Weisfeiler-Lehman test for graph isomorphism.

Similar to the second group, our theoretical results concern the capacity of aggregation functions. However, somewhat in contrast to what was previously shown, we put forth an exponential lower bound on the output size of any (even approximately) injective aggregation function. This provides evidence that the maximum capacity of GNN might be unattainable in practice. We then provide an alternative notion of capacity and show that it can be more pragmatically satisfied.

We should also mention the relevant result of~\cite{levie2019transferability}, who studied the stability of spectral graph convolution. By bounding the effect that changing the graph might have, this work indirectly poses a bound on discrimination capacity of a specific permutation equivariant layer. Placing this result in context of our theoretical framework could be an interesting further step.

\paragraph{Notation.} We use bold symbols to denote matrices and vectors, e.g., $\bX$ and $\bx$.  
We focus on weighted and possibly directed graphs $G = (V,E, w)$, where $V$ and $E$ are the node and edge sets, and $w_{ij}> 0$ only if there exist an edge $e_{ij} \in E$ between nodes $v_i$ and $v_j$.
Note that sets are distinguished from multisets (defined later) by using a calligraphic upper-case symbol: $\cX$ is a multiset and $X$ is a set. 
We also denote aggregation functions in a typewriter font (e.g., $\f$ and $\g$) in contrast to general functions (e.g., $f$ and $g$). 
We focus specifically on $d$-multisets: multisets whose elements $\bx \in \cX$ have the same dimension $ \bx \in \bbR^d$. 
We also assume that all vectors $\bx$ lie in some bounded domain $D \subset \bbR^d$.  

\Section{The discrimination capacity of aggregation functions}
\label{sec:aggregation_functions}

Suppose that we are interested in the features $\cX = \{ \bx_1, \cdots, \bx_n\}$ supported over a set of $n$ nodes---the latter might correspond to the neighborhood of a node or even an entire graph. 
Lacking a consistent method of ordering the elements of $\cX$ (i.e., one that translates across different neighborhoods or graphs) we treat it as a \textit{multiset}---a set with (possibly) repeating elements: 
\begin{definition}
A multiset $\cX$ is a 2-tuple $(X, \mathsf{c})$, where $X$ is a base set and $\mathsf{c}\colon X \to \mathbb{N}_{\geq 1}$ is a function 
that counts the multiplicity
of each element $\bx$ from $X$ in $\cX$. 
\end{definition}

The inner workings of a number of learning algorithms for graphs can be interpreted as an application of a (possibly involved and parametrized) sequence of functions over multisets. 
We call each such operation an \textit{aggregation function}, because it involves the aggregation of information (features or hidden representations) over a multiset:
\begin{definition}
An {aggregation function} is a map from the set of multisets onto some vector space. 
\end{definition}

It can be immediately realized that every aggregation function must be \textit{invariant to permutation}.

As aptly reviewed by~\citep{gilmer2017neural}, aggregation functions are commonly employed to obtain a vector representation of the information in the neighborhood of a node or an entire graph (also referred to as global pooling or readout function): 
For instance, a number of variants of local aggregation can be built combining a sum with different types of linear layers (e.g., Laplacian eigenvector specific~\citep{bruna2013spectral,defferrard2016convolutional},  degree normalized~\citep{kipf2016semi}, degree dependent~\citep{duvenaud2015convolutional}, edge-label dependent~\citep{li2015gated}) or MLPs~\citep{kearnes2016molecular,schutt2017quantum}. Similarly, global aggregation usually merges the summing of node representations (possibly over many layers) with a feed-forward neural network.

\Subsection{A curse of dimensionality lower bound}  
\label{subsec:lower-bound}

Out of all aggregation functions, those that are \textit{injective}\footnote{A function $\f$ is injective if for any multisets $\cX,\cY$, we have $\f(\cX) = \f(\cY)$  \textit{only if} $\cX = \cY$.} possess the largest discrimination capacity. 
Thus, it might be tempting to attempt to maximize the discrimination power of graph neural networks by equipping them with injective aggregators. In theory, this would allow them to distinguish between multisets having the most subtle differences. 

In the following, we present evidence that this is not achievable in the context of learning from finite samples.
To realize our argument we consider a generalization of injectivity expliciting the discrimination precision $\delta$. This subsumes the classical definition (obtained for $\delta=0$) and allows us to reason about the discrimination capacity of functions with finite precision, such as neural networks: 

\begin{definition}[$\delta$-injective function]
An aggregation function $\f$ is called $\delta$\textit{-injective} (within some domain $D$) for some $\delta \geq 0$ if, for any $d$-multisets $\cX$ and $\cY$, we have $\f(\cX) = \f(\cY)$ only if there exists a bijective map $\phi : \cX \to \cY$ with $\|\phi(\bx) - \by\|_p \leq \delta$ for all $\bx \in \cX$ and $\by \in \cY$. 
\end{definition}

The definition determines if two multisets are the same if there exists a bijective map between their (possibly infinite) elements. It can be therefore seen that, for any two multisets with different cardinalities the output representations must differ (since no bijective map exists between them). 
On the other hand, $\delta$ effectively controls the accuracy of discrimination---it could be for example set according to machine precision.

With this in place, our next step will be to establish a lower bound on the output size of any $\delta$-injective aggregation function. To circumvent pathological cases (due to uncountable sets), we will herein constrain each output variable of the functions we consider to the countable set $\mathbb{P} \subset \bbR$ of real numbers that can be computed to within any desired precision by a finite, terminating algorithm.\footnote{Such pathological cases will not appear in the rest of the paper, and thus, in the following parts we revert to the common practice of supposing that the output of a neural network lives in $\bbR^m$.}

\begin{theorem}
    Every $\delta$-injective aggregation function that maps $d$-multisets onto $\mathbb{P}^m$ must have 
    $$
        m \geq  \left( \frac{1}{2\delta} \right)^d \frac{\text{vol}(D)}{\text{vol}(B_p(1))}
    $$
    outputs, where $B_p(1)$ is the unit-norm ball in $\bbR^d$ w.r.t. the $\ell_p$-norm. 
    \label{prop:injective-bound}
\end{theorem}

(All proofs can be found in the supplementary material.) This means that the number of outputs of a $\delta$-injective function $\f$ should be at least $m \geq \left(2 \delta \right)^{-d}$ whenever $B_p(1) \subseteq D$. This is problematic as it is likely that the number of parameters to be learned (e.g., if $\f$ is approximated by an MLP as we will see in the following) would also need to depend exponentially on $d$.

\Subsection{Aggregation functions with bounded discrimination capacity}

In light of this negative result,  we turn to the following restricted desideratum: 

\begin{definition}[($t,\delta$)-injective function] An aggregation function $\f$ is called $(t, \delta)$-injective if, for any $d$-multisets  $\cX_1, \ldots, \cX_t$ there exists function $\varphi$, such that $\f(\{\varphi(\bx) : \bx \in \cX_i\}) = \f(\{\varphi(\bx) : \bx \in \cX_j\})$ only if there exists a bijective map $\phi : \cX_i \to \cX_j$ with $\|\phi(\bx) - \by\|_p \leq \delta$ for all $\bx \in \cX_i$ and $\by \in \cX_j$.
\end{definition}

Thus, here we are only interested in distinguishing between $t$ multisets. The main challenge is that this has to be true for \textit{any} such multisets. One may think for instance that $\cX_1, \ldots, \cX_t$ are chosen by an adversary. Then an aggregation function $\f$ is ($t,\delta$)-injective if it cannot be fooled into mapping different multisets to the same representation: there always exists some function $\varphi$ such that $\f(\{\varphi(\bx) : \bx \in \cX_i\}) \neq \f(\{\varphi(\bx) : \bx \in \cX_j\})$ for all $i \neq j.$

The benefit of this formulation is that, since $\varphi(\bx)$ acts independently on each $\bx \in \cX$, it can be approximated by an MLP with one hidden layer~\citep{hornik1989multilayer,cybenko1989approximation}.
Thus, the discrimination capacity of a ($t,\delta$)-injective function can be optimized by learning the MLP weights.

Next, we examine two $(t,\delta)$-injective functions and reason about their parameter complexity. 

\paragraph{Sum.} One of the most commonly encountered methods of aggregating features over a graph (or neighborhood) entails summing (a function of) the multiset elements:
\begin{align}
        \textsf{sum}(\cX) = \sum_{\bx \in \cX} \varphi(\bx), 
\end{align}
where function $\varphi$ in once more approximated by an MLP.
Though it is known that this function can be injective when the set of possible inputs (possible $d$-multisets) is countable \cite[Lemma 5]{xu2018powerful}, as we found out in Section~\ref{subsec:lower-bound}, the output dimension (and parameter complexity) of injective constructions has to be exponential in $d$. 
However, it turns out that the number of outputs decreases significantly if we only ask for bounded discrimination capacity:

\begin{theorem}
    If $\varphi(\bx)$ has $t$ outputs, $\textsf{sum}(\cX)$ is $(t,\delta)$-injective w.r.t. the $\ell_\infty$-norm.  
\label{prop:sum-mlp}
\end{theorem}

The above theorem suggests that the parameter complexity, i.e., the number of weights of the MLP, depends on the discrimination capacity $t$ of $\textsf{sum}$, but is independent of the maximum cardinality and dimension of multisets involved. 
This is an encouraging result for situations where we only need to distinguish between a few different types of multisets of possibly very large cardinality, such as for instance when we are aggregating the features of the neighbors of a given node in a scale-free graph.

\begin{figure*}
\centering
\begin{subfigure}{.42\textwidth}
  \centering
  \includegraphics[width=1.00\linewidth, trim=0mm 0mm 0mm 0mm, clip]{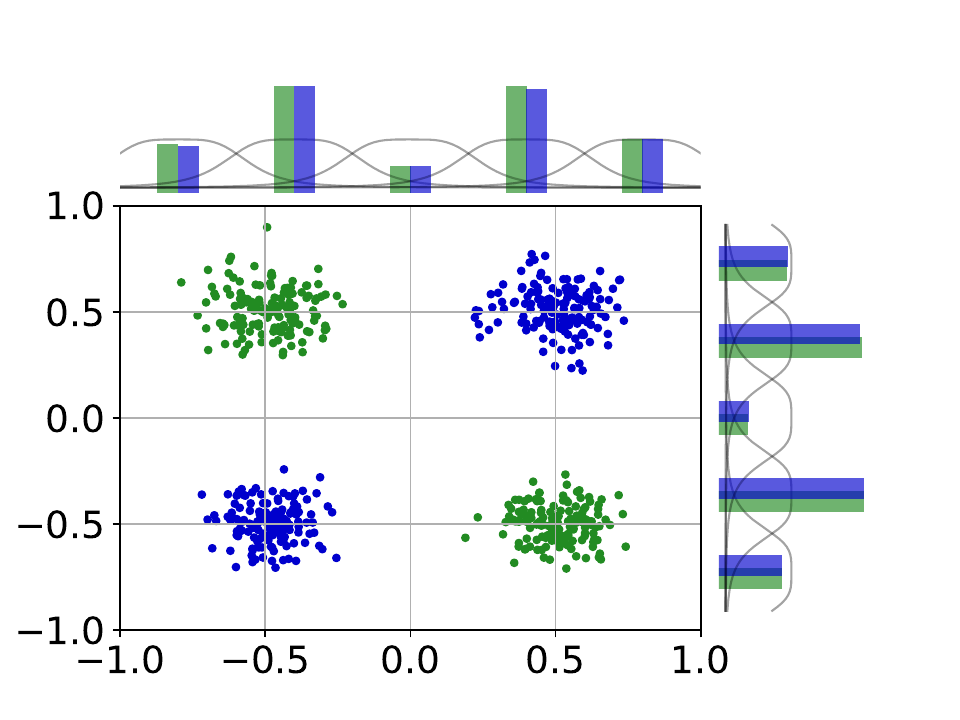}
  \label{fig:robustness:graph}
\end{subfigure}%
 \hfill
\begin{subfigure}{.42\textwidth}
    \centering
  \includegraphics[width=1.0\linewidth, trim=0mm 0mm 0mm 0, clip]{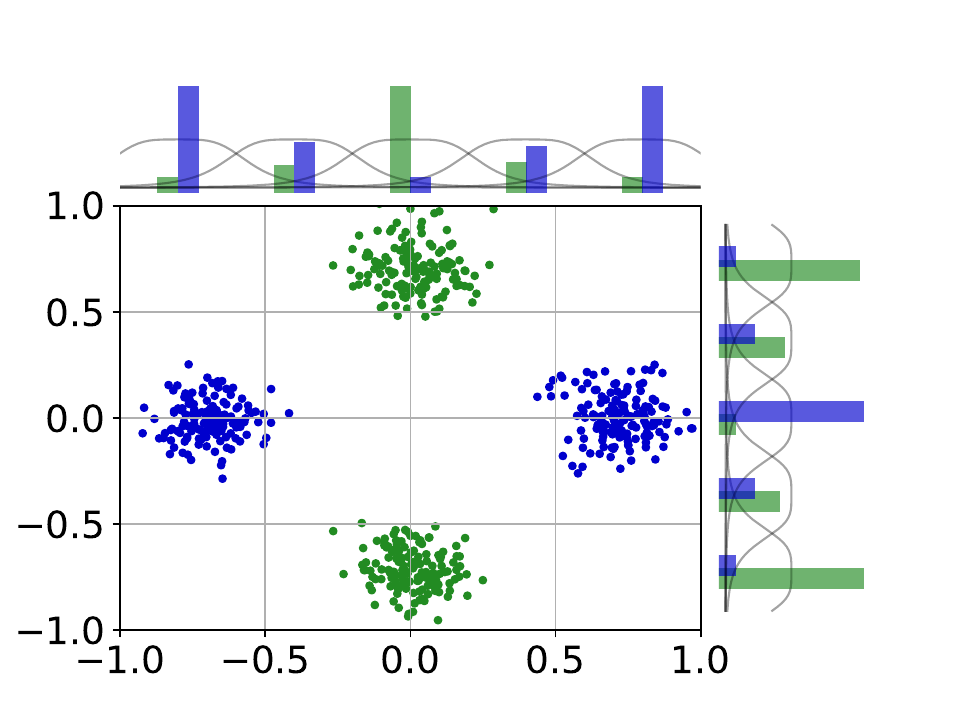}
  \label{fig:robustness:filters}
\end{subfigure}
\vspace{-5mm}
\caption{
\small With the projective histogram one computes a 1D histogram over each projected coordinate of a multiset (colored points). The translated kernels giving rise to bins are shown in grey and the colored bars correspond to the (normalized) $\p$ outputs. On the left it can be seen that without projection, $\p$ is blind to the differences of two different 2-multisets (green and blue). The aggregation function however can discriminate between them if we first employ a function $\varphi$ independently on each element. For example, on the right, the output of $\p$ becomes different when the multiset elements are rotated by 45 degrees. \vspace{-4mm} 
} 
\label{fig:phist}
\end{figure*}

\paragraph{Projective histogram.} We now propose an alternative tailored to situations where we need large discrimination capacity (i.e., $t \gg d$) without increasing the number of parameters.

The function in question summarizes a multiset by computing multiple one dimensional histrograms\footnote{A ``proper'' multi-dimensional histogram can be shown to be $\delta$-injective and has exponentially-many outputs.} of its projected elements.
Concretely, let $P = \{p_l = \frac{2l-1}{b}-1: l = 1, \cdots, b\}$ be a set of $b$ equidistant points in $[-1,1]$, corresponding to the bin centers.
The \textit{projective histogram} $\p(\cX) \in \bbR^{d \times b}$ function is defined as: 
\begin{align}
    [\p(\cX)]_{i,l} = \sum_{\bx \in \cX} \kappa(| [\varphi(\bx)]_i - p_l|), 
\end{align}
with $\kappa : \bbR_+ \to \bbR_+$ a kernel of width $w$ and function $\varphi$ having $d$ outputs. 
For a visual demonstration, we refer the reader to Figure~\ref{fig:phist}.
Though a new addition to graph neural network toolbox, histograms have been considered before in the context of deep learning~\citep{ustinova2016learning,chiu2015see,DBLP:journals/corr/abs-1804-09398}. In fact, for certain kernels it is possible to also learn the points $P$ as well as the kernel widths. 
However, as it is shown next, appropriately constructed projective histograms are $(b d,\delta)$-injective, even when the bins are fixed:

\begin{theorem}
    Let $\kappa$ be a uniform kernel of width $w = 1/b$ and fix $m = b d\geq t$ with $b = 2/ \delta + 1 > 1$. There exist at least $m!$ different functions $\varphi(\bx)$ with $d$ outputs such that $\p(\cX)$ is $(t,\delta)$-injective w.r.t. the $\ell_\infty$-norm.
\label{prop:proj-hist-mlp}
\end{theorem}

With $\p$ therefore, the function applied element-wise on each feature can have $d$ outputs (and not $t$ as with $\textsf{sum}$). This is beneficial in terms of parameter complexity when $\varphi$ is approximated by an MLP, as the number of weights to be learned remains independent of $t$. If the number of parameters is of no concern, then sums can be as powerful as projective histograms (in theory, sum aggregation followed by an MLP can be used to approximate any aggregation function~\citep{xu2018powerful}).

It is important to stress that the above analysis does {not} guarantee that any GNN will \textit{learn} to be $(t,\delta)$-injective. It is only a statement about the discrimination capacity of some aggregation functions. Nevertheless, it is encouraging that the appropriate parametrization of the MLP is far from unique: according to Theorem~\ref{prop:proj-hist-mlp}, for any $t$ multisets, each consisting of possibly a very large number of vectors in $\bbR^d$, there exist a very large number ($m!$) of functions $\varphi$, such that the respective projective histogram can distinguish between them. 

\paragraph{In practice.} We recommend using the sum and its multiple-hop variants (such as graph spectral convolution) for local aggregation and projective histograms for global pooling. For the latter, our experiments indicate that it can be advantageous to optimization to utilize partially overlapping and smooth kernels (as in Figure~\ref{fig:phist}) instead of uniform non-overlapping kernels. Though we lack a formal proof, we suspect that the $(t,\delta)$-injectivity guarantees also extend to this case.

\Section{A neural network for structural graph classification} 

Our goal is to find a generic way of learning discriminative graph representations that also generalize well, without relying on node or edge attributes. 
We opt for a two-level scheme, where the representation of a graph is constructed by aggregating those of its nodes, as follows:  
\begin{itemize}[itemindent=0cm, leftmargin=4mm]
    \item An embedding $\be_j$ is learned for every node $v_j$ aiming to capture its structural role in the graph:
    $$
        \be_j = f_\text{node}( \bdelta_j; G) \in \bbR^{m_\text{node}} \quad \text{for all} \quad v_j \in V,
    $$
    where $f_\text{node}$ is a graph neural network conditioned on the node's one-hot encoding $\bdelta_j$.
    
    \item A graph representation is then obtained by combining node embeddings in a manner that depends on their prevalence and inter-relation:
    $$
        \be_G = f_\text{graph}( \bE; G) \in \bbR^{m_\text{graph}} \quad \text{with} \quad  \bE =  [\be_1, \ldots, \be_n]^\top \in \bbR^{n \times m_\text{node}},
    $$
    where, from now on, $n=|V|$ corresponds to the number of nodes in $G$. 
\end{itemize}
Each representation above is learned by a separate \textit{structural embedding network} $f$. We stress that both embedding networks $f_\text{node}$ and $f_\text{graph}$ have the \textit{same} architecture (described in Section~\ref{subsec:embedding_network}) and are optimized \textit{jointly} in an end-to-end fashion, but they do not share parameters. We refer the reader to the supporting material for a schematic illustration of our neural network.

\Subsection{Structural embedding networks} 
\label{subsec:embedding_network}

Irrespective of whether they learn a node or graph representation, both structural embedding networks ($f_{\text{node}}$ and $f_{\text{graph}}$) output a vector of fixed dimension $m$ ($m_{\text{node}}$ and $m_{\text{graph}}$).

Let $\bx_i \in \bbR^{q}$ be an input feature vector associated with $v_i$ and $\bX$ the corresponding $n\times q$ matrix ($\bX$ corresponds to $\bdelta_j$ and $\bE$, respectively). 
We first pass $\bX$ through a graph convolutional network $g_\text{conv}(\,\cdot\,; G) : \bbR^{n\times q} \to \bbR^{n \times d}$ conditioned on $G$ in order to capture the short- or long-range inter-dependencies between nodes. The output of the GCN is written in multiset notation as  
$$
     \cX = \{ [g_\text{conv} (\bX; G)]_{i,:} ~~~  \text{for all} ~~ v_i \in V\},
$$ 
with $\cX$ having $n$ elements, each being the GCN output on a given node. Function $g_\text{conv}$ should contain multiple (at least two) layers of graph convolution, meaning that it repeatedly aggregates features over node neighborhoods. We chose to rely on parametrized spectral convolution layers~\citep{defferrard2016convolutional}, though other choices are also possible (e.g.,~\citep{levie2017cayleynets,DBLP:journals/corr/abs-1901-01343,gilmer2017neural,hamilton2017inductive}).

The structural embedding network then returns: 
$$
     f (\bX; G) = \bW \vec{\p(\cX)} + \bb \in \bbR^{m} \quad \text{with} \quad m = b d.
$$     
As discussed in Section~\ref{sec:aggregation_functions}, the discrimination capacity of $\p$ depends on the choice of a function $\varphi: \bbR^{d} \to \bbR^{d}$ aiming to locally transform the features of each node. In our design, this function is learned implicitly by the GCN: a $\tau$-layer GCN is at least as powerful as a $\tau$-layer perceptron and thus can serve as a universal function approximator whenever $\tau\geq 2$. 
Also, the final linear layer is redundant when $f (\bX; G)$ is used to compute an intermediate representation---we therefore add it only to the final structural embedding network $f_\text{graph}$.

\Subsection{Faster and more efficient training}
\label{subsec:tricks}

We would like to explicit two tricks that improve the training of our neural network. 

\paragraph{Multiset normalization.} Prior to feeding a multiset into $\p$, its elements should be properly normalized to lie in $[-1,1]$. A hyperbolic tangent $\sigma(t) = (e^{t} - e^{-t})/(e^{t} + e^{-t})$ suffices to carry out the normalization, but its use can also lead to vanishing gradient problems, especially when the network weights are improperly initialized. 
To deal with these issues, we took inspiration from batch normalization: interpreting the elements in the $i$-th coordinate of every $\bx \in \cX$ as a random variable with mean $\mu_i$ and variance $\sigma^2_i$, we perform the following whitening (before $\sigma$): 
$$
    [\hat{\bx}]_i = \frac{[\bx]_i}{\sigma_i} - \mu_i, 
$$
prior to utilizing a projective histogram. For the particular case of $f_\text{node}$, the statistics are computed jointly over the multisets of all nodes in the same graph. Though the mean and variance vectors can also be concatenated in the output of $\p(\hat{\cX})$, we did not choose to do so in our implementation as we could not identify any empirical benefits.

\paragraph{Implicit regularization by sampling.} 
To ease computational and memory requirements, we approximate $\be_G$ only w.r.t. a subset of node embeddings sampled (with replacement) from $V$, typically 32. We then rely on the graph embedding network (specifically on $g_\text{conv}$) to infer the missing node embeddings. At test time, the classification variance is controlled by averaging the logits over multiple realizations (we choose 10).

Though initially conceived for computational reasons, we discovered that the randomness introduced by sampling often helped to combat overfitting. In fact, for small graphs, we obtained better results by oversampling the node set (i.e., by sampling each node embedding multiple times and adding the resulting vectors) rather than considering every node exactly once.    
Our hypothesis is that, since for the same graph the input of the graph embedding network can differ depending on the sampling realizations, the graph embedding network cannot rely too much on specific node embeddings. In addition, the node embedding network is given an incentive to learn node embeddings that the graph embedding network can successfully interpolate. This constrains its capacity and can be beneficial to generalization (as confirmed by our experiments).   
We should note that, though sampling has also been used before to accelerate aggregation functions (see e.g., ~\citep{7383743, hamilton2017inductive}), we are not aware of any previous works using it to reduce overfitting.

\Subsection{Additional considerations}

In the following, we discuss and motivate some key aspects of our architecture: 

\paragraph{On the use of one-hot representations.} The use of $\bdelta_j$ as an input to the node embedding network $f_\text{node}$ and, as a consequence, the graph convolutional network $g_\text{conv}$ may initially come across as simplistic. It is however informative of the role of a node w.r.t. its surroundings. It has also been used successfully in the past~\citep{DBLP:journals/corr/ParattePV17,donnat2018learning}---though the latter approaches used hand-crafted convolution kernels and did not rely on learning. In fact, an interesting parallel can be drawn to system identification: quantity $g_\text{conv}(\bdelta_j, G)$ can be seen as the parametrized impulse response of a black box system defined by the graph and the convolution kernel centered at $v_j$. A successful parametrization entails rendering the impulse response sensitive to those properties of a graph relevant to the classification task, while at the same time being robust to errors in the graph structure, e.g., induced by noise or introduced during graph construction.

\paragraph{Why convolve node embeddings?} One of the design differences of our architecture with standard GNN is that, instead of directly aggregating node embeddings with a global pooling function, we first feed them to a convolutional network $g_\text{conv}(\, \cdot\,; G)$. 
To understand the intention behind this step, let us consider a toy example in which $f_\text{node}$ outputs only two node embeddings: a black and a white one (symbolically).
In this setting, independently of which aggregation function is used, graphs having the same number of black and white nodes will be indistinguishable (by the permutation invariance of aggregation functions). On the contrary, by using $g_\text{conv}$ we aim to make the graph embedding sensitive to {how} the black and white nodes are spatially distributed. For instance, a GCN can distinguish between the cases where colored nodes are randomly dispersed and are distributed in color-coherent clusters (by smoothing the embeddings).
In other words, convolving node embeddings allows us to learn \textit{non-global} node embeddings (by fixing the receptive field of the node embedding network to be smaller than the graph diameter), while still considering global information in the graph representation. Altenatively, to achieve a similar effect one should learn node embeddings that discriminate nodes over the entire graph.  

\paragraph{Computational complexity.} Each structural embedding network takes time linear to the number of nodes and edges, number of graph convolution layers, and number of bins. This implies that the \textit{exact} computation of $\be_G$ is quadratic w.r.t. the number of nodes. Luckily, as previously discussed, it can be beneficial to compute $\be_G$ on the basis of a (random) subset of nodes of constant size, in which case the end-to-end computational complexity for graph embedding remains linear.

\Section{Empirical results}

We evaluate our method by comparing it against 14 other methods (6 graph kernels and 8 GNN) in 9 benchmark datasets, and by performing ablation studies that aim to test the effect of specific architecture choices.

\paragraph{Datasets.}
We use 5 bioinformatics datasets: MUTAG, ENZYMES, PROTEINS, DD, NCII and 4 social network datasets: IMDB-B(INARY), IMDB-M(ULTI), COLLAB, REDDIT-B(INARY).\footnote{Download from \url{https://ls11-www.cs.tu-dortmund.de/staff/morris/graphkerneldatasets}.} 
We insist that in our setting only the graph structure is given as input with no additional node or edge features. As a result, we intentionally discarded any attributes, when present.

\paragraph{Baselines.}
We use the graph kernels implemented by the Grakel library~\citep{siglidis2018grakel}, where we replace all attributes with zeros. The graph kernel baselines are:
WL/Subtree, 
WL/Shortest-Path~\citep{shervashidze2011weisfeiler}, 
SvmTheta~\citep{johansson2014global},
ShortestPath~\citep{borgwardt2005shortest},
PyramidMatch~\citep{nikolentzos2017matching}, and
Graphlet Sampling~\citep{prvzulj2007biological,shervashidze2009efficient}.
We use the default parameters for all methods, except WP/Shortest-Path and Graphlet Sampling, where we decrease the iterations/samples to address computational issues.

For the GNN baselines, we report results from the pytorch geometric benchmark~\citep{Fey/Lenssen/2019}. For attributed datasets, we remove all attributes, use as input features a one-hot encoding of the node degrees (with a maximum value of 40), and rerun the benchmark. The GNN baselines are:
GCN~\citep{kipf2016semi},
GIN-0,
GIN-$\epsilon$~\citep{xu2018powerful},
GraphSAGE, 
GraphSAGE w/o JK~\citep{hamilton2017inductive},
GlobalAttention~\citep{li2015gated},
Graclus (coarsening based)~\citep{dhillon2007weighted}, and
Set2SetNet~\citep{vinyals2015order}.

For all methods, we use 10-fold stratified cross-validation and report the mean/std-dev of the test accuracy. For recomputed GNN benchmarks, we increase the default number of epochs to 400 to ensure convergence and perform a hyper-parameter search over the number of layers (2, 3, 4, 5) and the number of channels (16, 32, 64, 128).

\paragraph{Experiment settings.} Both structural embedding networks of our base architecture are composed of two $d \in \{16, 24, 32\}$-channel spectral convolution layers of order $\in \{4,5,6,7\}$ followed by a projective histogram with $b \in \{5, 8, 10\}$ bins (the exact values depend on the dataset). For training, we use the Adam optimizer and a batch size of $\in \{8,20\}$. The exact hyperparameters are listed in the supplementary material. 
We also tested two ablated variants of our architecture: the first uses a sum instead of a projective histogram, and the second is trained with a deterministic input, i.e., $f_{\text{graph}}$ does not rely on the embeddings of a randomly sampled node-set, but uses the embeddings of all nodes. In the later case, due to computational constraints we only considered a subset of the datasets.

\begin{wraptable}{r}{0.6\textwidth}
\vspace{-3mm}
\setlength\tabcolsep{1.6mm}
\resizebox{0.6\textwidth}{!}{%
\begin{tabular}{rcccccc}
\toprule
                 & MUTAG  & ENZYMES  & IMDB-B  & IMDB-M & PROTEINS  & NCI1  \\ 
\midrule
with sampling  & 4.2   & 22.8    & 7.8     & 7.2   & 8.1      &  12.4              \\ 
deterministic    & 8.1   & 62.6    & 9.0     & 11.5  & 25.8     &  19.6              \\ 
\bottomrule
\vspace{-2mm}
\end{tabular}%
}
\caption{{\small \textbf{Generalization error}. We report the difference of classification error (in percentage) between the training and the testing set on a subset of (smaller) datasets. 
}}
\vspace{-3mm}
\label{tab:gen-error}
\end{wraptable}

\paragraph{Results.}
Table~\ref{tab:main-result} presents our results. 
In the absence of attributes, GNN baselines can often be seen to be inferior to graph kernels---which we believe is due to overfitting. For richer datasets (DD, COLLAB and REDDIT-B), GNN become more efficient as they can leverage their discrimination power to a better effect.
In contrast, our architecture attains competitive accuracy in most datasets, often exceeding the state-of-the-art.
One of the key elements of our network is random sampling, which we argue reduces overfitting by acting as an implicit regularization. 
This claim is supported by the poor performance of the deterministic variant of our architecture and by the observation that randomness reduces significantly the generalization error---reported in Table~\ref{tab:gen-error}.  
Using a projective histogram can also lead to a modest accuracy increase, with the difference being particularly prominent for the ENZYMES dataset. For the latter, we noticed that most methods were prone to severe overfitting (see e.g., Table~\ref{tab:gen-error}). Moreover, for certain datasets the performance of baselines dropped significantly as compared to the results attained when attributes are available~\citep{Fey/Lenssen/2019}, suggesting that, though these algorithms might be efficient at utilizing attributes, they can sometimes struggle to classify graphs based purely on structure.

\begin{table}[]
\centering
\setlength\tabcolsep{1.6mm}
\resizebox{\textwidth}{!}{%
\begin{tabular}{rccccccccc}
\toprule
                                & MUTAG               & ENZYMES            & IMDB-B*             & IMDB-M             & PROTEINS           &  DD                & NCI1               & COLLAB*             & REDDIT-B*  \\ 
\midrule
WL/Subtree                      & 75.9 $\pm$ \, 8.6   & 23.2 $\pm$ 4.2     & 72.5 $\pm$ 4.2     & \b{50.7 $\pm$ 5.9} & 69.4 $\pm$ 3.7     & 58.7 $\pm$ 0.2     & 71.5 $\pm$ 2.1     & 77.7 $\pm$ 2.2     & 66.8 $\pm$ 5.2 \\
WL/Shortest-Path                & 83.3 $\pm$ 10.7     & 25.8 $\pm$ 5.7     & 72.1 $\pm$ 3.8     & 50.3 $\pm$ 5.5     & \b{74.2 $\pm$ 2.5} & 75.9 $\pm$ 2.4     & 66.3 $\pm$ 3.0     & 71.4 $\pm$ 2.2     & 83.2 $\pm$ 2.6 \\
Svm Theta                       & 83.8 $\pm$ 11.4     & 21.7 $\pm$ 5.2     & 52.3 $\pm$ 3.7     & 39.4 $\pm$ 5.1     & 72.1 $\pm$ 3.0     & 60.5 $\pm$ 0.9     & 62.2 $\pm$ 2.6     & 52.0 $\pm$ 0.1     & 64.5 $\pm$ 3.1 \\
Shortest Path                   & 80.5 $\pm$ 11.1     & 25.7 $\pm$ 4.8     & 58.4 $\pm$ 3.3     & 39.8 $\pm$ 2.9     & 72.0 $\pm$ 4.2     & 75.5 $\pm$ 2.2     & 61.2 $\pm$ 1.8     & 58.8 $\pm$ 1.4     & 75.8 $\pm$ 2.8 \\
Pyramid Match                   & 84.9 $\pm$ 11.0     & 26.2 $\pm$ 6.1     & 66.2 $\pm$ 5.2     & 46.2 $\pm$ 4.9     & 72.1 $\pm$ 3.2     & 76.3 $\pm$ 3.6     & 65.4 $\pm$ 2.1     & 68.0 $\pm$ 2.5     & 78.0 $\pm$ 4.3 \\
Graphlet Sampling               & 72.7 $\pm$ 10.8     & 15.8 $\pm$ 3.3     & 59.9 $\pm$ 3.8     & 41.2 $\pm$ 3.8     & 69.8 $\pm$ 3.6     & 64.5 $\pm$ 3.8     & 59.0 $\pm$ 2.0     & 62.6 $\pm$ 1.9     & 66.2 $\pm$ 2.2 \\ 
\midrule
GCN                             & 78.1 $\pm$ 12.1     & 24.2 $\pm$ 4.7     & 72.6 $\pm$ 4.5     & 47.9 $\pm$ 4.5     & 67.2 $\pm$ 3.0     & 66.8 $\pm$ 4.3     & 59.9 $\pm$ 1.6     & \b{80.6 $\pm$ 2.1} & 89.3 $\pm$ 3.3 \\ 
GIN-0                           & 69.6 $\pm$ 10.6     & 21.5 $\pm$ 8.7     & 72.8 $\pm$ 4.5     & 48.5 $\pm$ 5.4     & 69.8 $\pm$ 2.5     & 68.1 $\pm$ 4.0     & 66.8 $\pm$ 3.2     & 79.3 $\pm$ 2.7     & 89.6 $\pm$ 2.6 \\ 
GIN-$\epsilon$                  & 70.5 $\pm$ 11.9     & 23.5 $\pm$ 5.7     & 72.1 $\pm$ 5.1     & 47.7 $\pm$ 4.6     & 65.2 $\pm$ 2.7     & 67.8 $\pm$ 4.3     & 65.3 $\pm$ 3.6     & 79.8 $\pm$ 2.4     & 90.3 $\pm$ 3.0 \\ 
SAGE                            & 79.8 $\pm$ 13.9     & 24.2 $\pm$ 4.5     & 72.4 $\pm$ 3.6     & 49.9 $\pm$ 5.0     & 65.9 $\pm$ 2.7     & 65.8 $\pm$ 4.9     & 59.4 $\pm$ 2.4     & 79.7 $\pm$ 1.7     & 89.1 $\pm$ 1.9 \\ 
SAGE w/o JK                     & 80.4 $\pm$ 15.5     & 23.5 $\pm$ 3.8     & 72.1 $\pm$ 4.4     & 49.6 $\pm$ 4.7     & 66.0 $\pm$ 2.5     & 65.9 $\pm$ 2.9     & 59.2 $\pm$ 2.3     & 79.6 $\pm$ 2.4     & 87.9 $\pm$ 1.9 \\ 
Attention                       & 77.1 $\pm$ 11.0     & 23.2 $\pm$ 3.2     & 72.3 $\pm$ 3.8     & 48.6 $\pm$ 5.0     & 66.7 $\pm$ 4.0     & 65.7 $\pm$ 4.1     & 57.2 $\pm$ 2.6     & 79.6 $\pm$ 2.2     & 87.4 $\pm$ 2.5 \\ 
Graclus                         & 75.4 $\pm$ 11.6     & 23.7 $\pm$ 4.5     & 72.2 $\pm$ 4.2     & 48.5 $\pm$ 5.5     & 64.1 $\pm$ 2.0     & 66.6 $\pm$ 4.1     & 59.9 $\pm$ 2.3     & 79.6 $\pm$ 2.0     & 88.8 $\pm$ 3.2 \\ 
Set2Set                         & 76.6 $\pm$ 15.1     & 22.2 $\pm$ 1.8     & 72.2 $\pm$ 4.2     & 49.0 $\pm$ 4.3     & 66.0 $\pm$ 3.2     & 65.8 $\pm$ 4.0     & 59.5 $\pm$ 2.2     & 79.6 $\pm$ 2.3     & 89.6 $\pm$ 2.4 \\ 
\midrule
ours                            & 86.7 $\pm$ 7.6     & \b{42.0 $\pm$ 7.2} & \b{73.2 $\pm$ 4.9} & 48.5 $\pm$ 4.8     & \b{74.2 $\pm$ 3.8}  & \b{77.4 $\pm$ 6.4} & \b{79.8 $\pm$ 1.2} & 79.2 $\pm$ 1.6     & \b{92.2 $\pm$ 2.4} \\ 
ours (sum)                      & 86.7 $\pm$ 7.7     & 33.5 $\pm$ 7.4     & 72.8 $\pm$ 4.6     & 47.3 $\pm$ 4.8     & 72.9 $\pm$ 3.2      & 76.4 $\pm$ 5.0     & 74.4 $\pm$ 1.6     & 76.6 $\pm$ 2.4     & 91.0 $\pm$ 2.8 \\ 
ours (deterministic)            & \b{86.9 $\pm$ 5.4} & 36.0 $\pm$ 7.1     & 68.1 $\pm$ 3.4     & 39.4 $\pm$ 3.7     & 68.3 $\pm$ 2.6     & --        & 78.2 $\pm$ 1.6     & --       & --    \\ 
%
\bottomrule
\vspace{0mm}
\end{tabular}%
}
\caption{{ \small
\textbf{Test set classification accuracy (in percentage)}. We report the mean and standard deviation over a 10-fold stratified cross validation. Best results appear in bold. Datasets annotated with (*) by-themselves do not have attributes and the results were reported by~\citet{Fey/Lenssen/2019}. 
\vspace{-0.7cm}
}}
\label{tab:main-result}
\end{table}

\Section{Conclusion}

This paper focused on the discrimination capacity of aggregation functions. From a theoretical perspective, we argued that known limits of capacity are unlikely to be achievable in practice; we also proposed a more restricted notion of capacity that does not suffer from a dimensionality curse. %
Turning then to the practical problem of structural graph classification, we proposed a graph neural network aiming, not to maximize discrimination capacity, but to learn discriminative representations that generalize well. We pursued this by reducing the number of trainable parameters and incorporating implicit regularization by random sampling.      
Our empirical comparison suggests that these choices can yield benefits in practice.

\small
\bibliographystyle{unsrtnat}
\bibliography{bibliography}

\begin{thebibliography}{38}
\providecommand{\natexlab}[1]{#1}
\providecommand{\url}[1]{\texttt{#1}}
\expandafter\ifx\csname urlstyle\endcsname\relax
  \providecommand{\doi}[1]{doi: #1}\else
  \providecommand{\doi}{doi: \begingroup \urlstyle{rm}\Url}\fi

\bibitem[{Scarselli} et~al.(2009){Scarselli}, {Gori}, {Tsoi}, {Hagenbuchner},
  and {Monfardini}]{4700287}
F.~{Scarselli}, M.~{Gori}, A.~C. {Tsoi}, M.~{Hagenbuchner}, and
  G.~{Monfardini}.
\newblock The graph neural network model.
\newblock \emph{IEEE Transactions on Neural Networks}, 20\penalty0
  (1):\penalty0 61--80, Jan 2009.
\newblock ISSN 1045-9227.
\newblock \doi{10.1109/TNN.2008.2005605}.

\bibitem[Gilmer et~al.(2017)Gilmer, Schoenholz, Riley, Vinyals, and
  Dahl]{gilmer2017neural}
Justin Gilmer, Samuel~S Schoenholz, Patrick~F Riley, Oriol Vinyals, and
  George~E Dahl.
\newblock Neural message passing for quantum chemistry.
\newblock In \emph{Proceedings of the 34th International Conference on Machine
  Learning-Volume 70}, pages 1263--1272. JMLR. org, 2017.

\bibitem[Hamilton et~al.(2017)Hamilton, Ying, and
  Leskovec]{hamilton2017inductive}
Will Hamilton, Zhitao Ying, and Jure Leskovec.
\newblock Inductive representation learning on large graphs.
\newblock In \emph{Advances in Neural Information Processing Systems}, pages
  1024--1034, 2017.

\bibitem[Zaheer et~al.(2017)Zaheer, Kottur, Ravanbakhsh, Poczos, Salakhutdinov,
  and Smola]{NIPS2017_6931}
Manzil Zaheer, Satwik Kottur, Siamak Ravanbakhsh, Barnabas Poczos, Ruslan~R
  Salakhutdinov, and Alexander~J Smola.
\newblock Deep sets.
\newblock In I.~Guyon, U.~V. Luxburg, S.~Bengio, H.~Wallach, R.~Fergus,
  S.~Vishwanathan, and R.~Garnett, editors, \emph{Advances in Neural
  Information Processing Systems 30}, pages 3391--3401. Curran Associates,
  Inc., 2017.
\newblock URL \url{http://papers.nips.cc/paper/6931-deep-sets.pdf}.

\bibitem[Xu et~al.(2018)Xu, Hu, Leskovec, and Jegelka]{xu2018powerful}
Keyulu Xu, Weihua Hu, Jure Leskovec, and Stefanie Jegelka.
\newblock How powerful are graph neural networks?
\newblock \emph{arXiv preprint arXiv:1810.00826}, 2018.

\bibitem[Duvenaud et~al.(2015)Duvenaud, Maclaurin, Iparraguirre, Bombarell,
  Hirzel, Aspuru-Guzik, and Adams]{duvenaud2015convolutional}
David~K Duvenaud, Dougal Maclaurin, Jorge Iparraguirre, Rafael Bombarell,
  Timothy Hirzel, Al{\'a}n Aspuru-Guzik, and Ryan~P Adams.
\newblock Convolutional networks on graphs for learning molecular fingerprints.
\newblock In \emph{Advances in neural information processing systems}, pages
  2224--2232, 2015.

\bibitem[Weisfeiler and Lehman(1968)]{weisfeiler1968reduction}
Boris Weisfeiler and Andrei~A Lehman.
\newblock A reduction of a graph to a canonical form and an algebra arising
  during this reduction.
\newblock \emph{Nauchno-Technicheskaya Informatsia (in russian)}, 2\penalty0
  (9):\penalty0 12--16, 1968.

\bibitem[Weisfeiler(2006)]{weisfeiler2006construction}
Boris Weisfeiler.
\newblock \emph{On construction and identification of graphs}, volume 558.
\newblock Springer, 2006.

\bibitem[Shervashidze et~al.(2011)Shervashidze, Schweitzer, Leeuwen, Mehlhorn,
  and Borgwardt]{shervashidze2011weisfeiler}
Nino Shervashidze, Pascal Schweitzer, Erik Jan~van Leeuwen, Kurt Mehlhorn, and
  Karsten~M Borgwardt.
\newblock Weisfeiler-lehman graph kernels.
\newblock \emph{Journal of Machine Learning Research}, 12\penalty0
  (Sep):\penalty0 2539--2561, 2011.

\bibitem[Maron et~al.(2019{\natexlab{a}})Maron, Hamu, Shamir, and
  Lipman]{DBLP:journals/corr/abs-1812-09902}
Haggai Maron, Heli~Ben Hamu, Nadav Shamir, and Yaron Lipman.
\newblock Invariant and equivariant graph networks.
\newblock ICLR, 2019{\natexlab{a}}.

\bibitem[Maron et~al.(2019{\natexlab{b}})Maron, Fetaya, Segol, and
  Lipman]{maron2019universality}
Haggai Maron, Ethan Fetaya, Nimrod Segol, and Yaron Lipman.
\newblock On the universality of invariant networks.
\newblock \emph{arXiv preprint arXiv:1901.09342}, 2019{\natexlab{b}}.

\bibitem[Keriven and Peyr{\'e}(2019)]{keriven2019universal}
Nicolas Keriven and Gabriel Peyr{\'e}.
\newblock Universal invariant and equivariant graph neural networks.
\newblock \emph{arXiv preprint arXiv:1905.04943}, 2019.

\bibitem[Levie et~al.(2019)Levie, Isufi, and
  Kutyniok]{levie2019transferability}
Ron Levie, Elvin Isufi, and Gitta Kutyniok.
\newblock On the transferability of spectral graph filters.
\newblock \emph{arXiv preprint arXiv:1901.10524}, 2019.

\bibitem[Bruna et~al.(2013)Bruna, Zaremba, Szlam, and LeCun]{bruna2013spectral}
Joan Bruna, Wojciech Zaremba, Arthur Szlam, and Yann LeCun.
\newblock Spectral networks and locally connected networks on graphs.
\newblock \emph{arXiv preprint arXiv:1312.6203}, 2013.

\bibitem[Defferrard et~al.(2016)Defferrard, Bresson, and
  Vandergheynst]{defferrard2016convolutional}
Micha{\"e}l Defferrard, Xavier Bresson, and Pierre Vandergheynst.
\newblock Convolutional neural networks on graphs with fast localized spectral
  filtering.
\newblock In \emph{Advances in neural information processing systems}, pages
  3844--3852, 2016.

\bibitem[Kipf and Welling(2016)]{kipf2016semi}
Thomas~N Kipf and Max Welling.
\newblock Semi-supervised classification with graph convolutional networks.
\newblock \emph{arXiv preprint arXiv:1609.02907}, 2016.

\bibitem[Li et~al.(2016)Li, Tarlow, Brockschmidt, and Zemel]{li2015gated}
Yujia Li, Daniel Tarlow, Marc Brockschmidt, and Richard Zemel.
\newblock Gated graph sequence neural networks.
\newblock ICLR, 2016.

\bibitem[Kearnes et~al.(2016)Kearnes, McCloskey, Berndl, Pande, and
  Riley]{kearnes2016molecular}
Steven Kearnes, Kevin McCloskey, Marc Berndl, Vijay Pande, and Patrick Riley.
\newblock Molecular graph convolutions: moving beyond fingerprints.
\newblock \emph{Journal of computer-aided molecular design}, 30\penalty0
  (8):\penalty0 595--608, 2016.

\bibitem[Sch{\"u}tt et~al.(2017)Sch{\"u}tt, Arbabzadah, Chmiela, M{\"u}ller,
  and Tkatchenko]{schutt2017quantum}
Kristof~T Sch{\"u}tt, Farhad Arbabzadah, Stefan Chmiela, Klaus~R M{\"u}ller,
  and Alexandre Tkatchenko.
\newblock Quantum-chemical insights from deep tensor neural networks.
\newblock \emph{Nature communications}, 8:\penalty0 13890, 2017.

\bibitem[Hornik et~al.(1989)Hornik, Stinchcombe, and
  White]{hornik1989multilayer}
Kurt Hornik, Maxwell Stinchcombe, and Halbert White.
\newblock Multilayer feedforward networks are universal approximators.
\newblock \emph{Neural networks}, 2\penalty0 (5):\penalty0 359--366, 1989.

\bibitem[Cybenko(1989)]{cybenko1989approximation}
George Cybenko.
\newblock Approximation by superpositions of a sigmoidal function.
\newblock \emph{Mathematics of control, signals and systems}, 2\penalty0
  (4):\penalty0 303--314, 1989.

\bibitem[Ustinova and Lempitsky(2016)]{ustinova2016learning}
Evgeniya Ustinova and Victor Lempitsky.
\newblock Learning deep embeddings with histogram loss.
\newblock In \emph{Advances in Neural Information Processing Systems}, pages
  4170--4178, 2016.

\bibitem[Chiu and Fritz(2015)]{chiu2015see}
Wei-Chen Chiu and Mario Fritz.
\newblock See the difference: Direct pre-image reconstruction and pose
  estimation by differentiating hog.
\newblock In \emph{Proceedings of the IEEE International Conference on Computer
  Vision}, pages 468--476, 2015.

\bibitem[Wang et~al.(2018)Wang, Li, Ouyang, and
  Wang]{DBLP:journals/corr/abs-1804-09398}
Zhe Wang, Hongsheng Li, Wanli Ouyang, and Xiaogang Wang.
\newblock Learnable histogram: Statistical context features for deep neural
  networks.
\newblock \emph{CoRR}, abs/1804.09398, 2018.
\newblock URL \url{http://arxiv.org/abs/1804.09398}.

\bibitem[Levie et~al.(2017)Levie, Monti, Bresson, and
  Bronstein]{levie2017cayleynets}
Ron Levie, Federico Monti, Xavier Bresson, and Michael~M Bronstein.
\newblock Cayleynets: Graph convolutional neural networks with complex rational
  spectral filters.
\newblock \emph{IEEE Transactions on Signal Processing}, 67\penalty0
  (1):\penalty0 97--109, 2017.

\bibitem[Bianchi et~al.(2019)Bianchi, Grattarola, Livi, and
  Alippi]{DBLP:journals/corr/abs-1901-01343}
Filippo~Maria Bianchi, Daniele Grattarola, Lorenzo Livi, and Cesare Alippi.
\newblock Graph neural networks with convolutional {ARMA} filters.
\newblock \emph{CoRR}, abs/1901.01343, 2019.
\newblock URL \url{http://arxiv.org/abs/1901.01343}.

\bibitem[{Isufi} et~al.(2015){Isufi}, {Simonetto}, {Loukas}, and
  {Leus}]{7383743}
E.~{Isufi}, A.~{Simonetto}, A.~{Loukas}, and G.~{Leus}.
\newblock Stochastic graph filtering on time-varying graphs.
\newblock In \emph{2015 IEEE 6th International Workshop on Computational
  Advances in Multi-Sensor Adaptive Processing (CAMSAP)}, pages 89--92, Dec
  2015.
\newblock \doi{10.1109/CAMSAP.2015.7383743}.

\bibitem[Paratte et~al.(2017)Paratte, Perraudin, and
  Vandergheynst]{DBLP:journals/corr/ParattePV17}
Johan Paratte, Nathana{\"{e}}l Perraudin, and Pierre Vandergheynst.
\newblock Compressive embedding and visualization using graphs.
\newblock \emph{CoRR}, abs/1702.05815, 2017.
\newblock URL \url{http://arxiv.org/abs/1702.05815}.

\bibitem[Donnat et~al.(2018)Donnat, Zitnik, Hallac, and
  Leskovec]{donnat2018learning}
Claire Donnat, Marinka Zitnik, David Hallac, and Jure Leskovec.
\newblock Learning structural node embeddings via diffusion wavelets.
\newblock In \emph{Proceedings of the 24th ACM SIGKDD International Conference
  on Knowledge Discovery \& Data Mining}, pages 1320--1329. ACM, 2018.

\bibitem[Siglidis et~al.(2018)Siglidis, Nikolentzos, Limnios, Giatsidis,
  Skianis, and Vazirgiannis]{siglidis2018grakel}
Giannis Siglidis, Giannis Nikolentzos, Stratis Limnios, Christos Giatsidis,
  Konstantinos Skianis, and Michalis Vazirgiannis.
\newblock Grakel: A graph kernel library in python.
\newblock \emph{arXiv preprint arXiv:1806.02193}, 2018.

\bibitem[Johansson et~al.(2014)Johansson, Jethava, Dubhashi, and
  Bhattacharyya]{johansson2014global}
Fredrik Johansson, Vinay Jethava, Devdatt Dubhashi, and Chiranjib
  Bhattacharyya.
\newblock Global graph kernels using geometric embeddings.
\newblock In \emph{Proceedings of the 31st International Conference on Machine
  Learning, ICML 2014, Beijing, China, 21-26 June 2014}, 2014.

\bibitem[Borgwardt and Kriegel(2005)]{borgwardt2005shortest}
Karsten~M Borgwardt and Hans-Peter Kriegel.
\newblock Shortest-path kernels on graphs.
\newblock In \emph{Fifth IEEE international conference on data mining
  (ICDM'05)}, pages 8--pp. IEEE, 2005.

\bibitem[Nikolentzos et~al.(2017)Nikolentzos, Meladianos, and
  Vazirgiannis]{nikolentzos2017matching}
Giannis Nikolentzos, Polykarpos Meladianos, and Michalis Vazirgiannis.
\newblock Matching node embeddings for graph similarity.
\newblock In \emph{Thirty-First AAAI Conference on Artificial Intelligence},
  2017.

\bibitem[Przulj(2007)]{prvzulj2007biological}
Natasa Przulj.
\newblock Biological network comparison using graphlet degree distribution.
\newblock \emph{Bioinformatics}, 23\penalty0 (2):\penalty0 e177--e183, 2007.

\bibitem[Shervashidze et~al.(2009)Shervashidze, Vishwanathan, Petri, Mehlhorn,
  and Borgwardt]{shervashidze2009efficient}
Nino Shervashidze, SVN Vishwanathan, Tobias Petri, Kurt Mehlhorn, and Karsten
  Borgwardt.
\newblock Efficient graphlet kernels for large graph comparison.
\newblock In \emph{Artificial Intelligence and Statistics}, pages 488--495,
  2009.

\bibitem[Fey and Lenssen(2019)]{Fey/Lenssen/2019}
Matthias Fey and Jan~E. Lenssen.
\newblock Fast graph representation learning with {PyTorch Geometric}.
\newblock In \emph{ICLR Workshop on Representation Learning on Graphs and
  Manifolds}, 2019.

\bibitem[{Dhillon} et~al.(2007){Dhillon}, {Guan}, and
  {Kulis}]{dhillon2007weighted}
I.~S. {Dhillon}, Y.~{Guan}, and B.~{Kulis}.
\newblock Weighted graph cuts without eigenvectors a multilevel approach.
\newblock \emph{IEEE Transactions on Pattern Analysis and Machine
  Intelligence}, 29\penalty0 (11):\penalty0 1944--1957, Nov 2007.
\newblock ISSN 0162-8828.
\newblock \doi{10.1109/TPAMI.2007.1115}.

\bibitem[Vinyals et~al.(2015)Vinyals, Bengio, and Kudlur]{vinyals2015order}
Oriol Vinyals, Samy Bengio, and Manjunath Kudlur.
\newblock Order matters: Sequence to sequence for sets.
\newblock \emph{arXiv preprint arXiv:1511.06391}, 2015.

\end{thebibliography}

\begin{appendices}
\section{Network architecture and hyperparameters}

\paragraph{Architecure.} We present a schematic illustration of our graph neural network architecture:

\begin{figure}[h!]
    \centering
    \includegraphics[width=1\columnwidth]{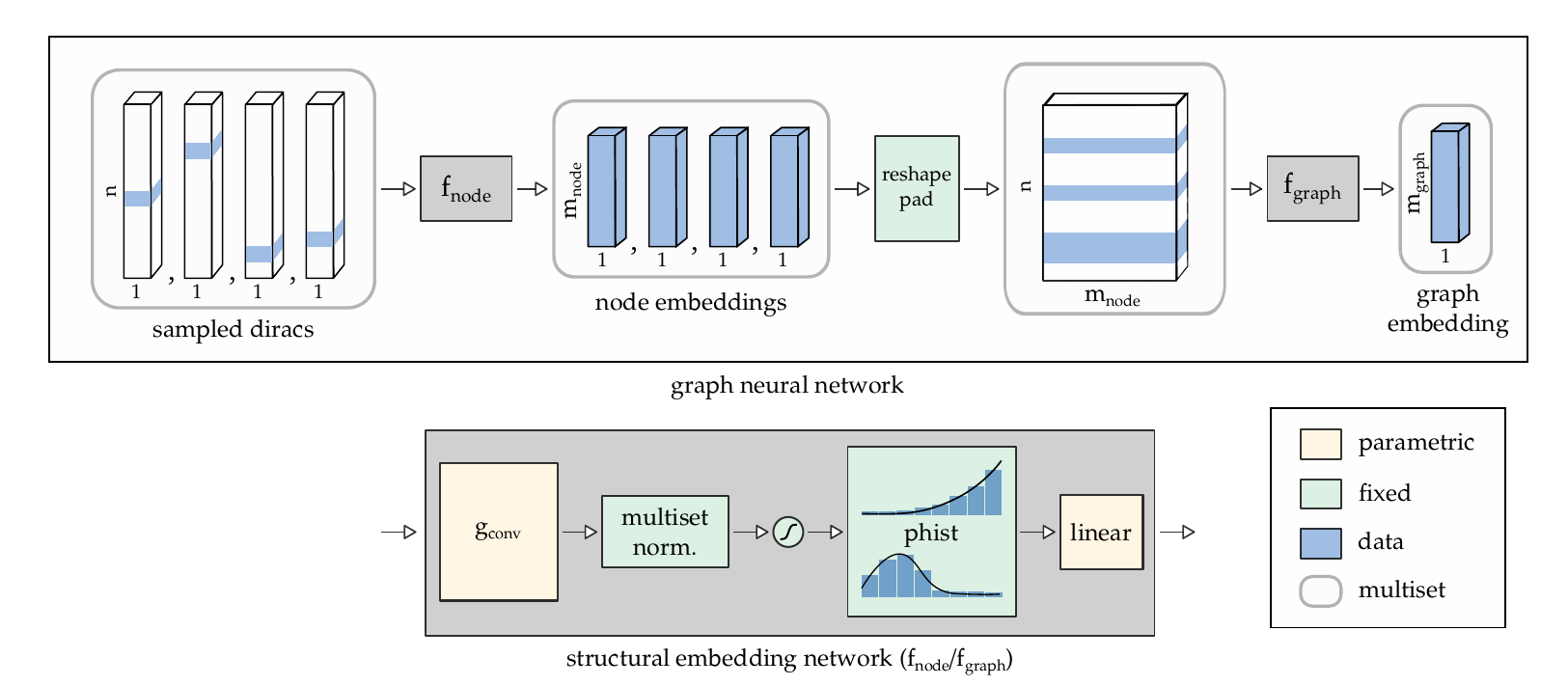}
\end{figure}

\paragraph{Hyper-parameters.} The hyper-parameters used with each dataset are:

\begin{tabular}{rccccccccc}
\toprule
                           & batch size      & samples &  order (1st/2nd layer)   & channels (1st/2nd layer)   & bins      \\ 
\midrule
MUTAG                      & 8      & 32        & (4,      4)  & (16,      24)    & 8  \\
ENZYMES                    & 20     & 32        & (3,      3)  & (16,      32)    & 5  \\
PROTEINS                   & 8      & 32        & (6,      6)  & (16,      32)    & 8  \\
DD                         & 20     & 50        & (3,      3)  & (16,      32)    & 8  \\
NCI1                       & 20     & 32        & (3,      3)  & (16,      32)    & 8  \\
IMDB-B                     & 20     & 32        & (3,      3)  & (16,      32)    & 8  \\
IMDB-M                     & 20     & 32        & (6,      6)  & (24,      32)    & 10 \\
COLLAB                     & 20     & 32        & (3,      3)  & (16,      32)    & 8  \\
REDDIT-B                   & 20     & 32        & (3,      3)  & (16,      32)    & 8  \\

\bottomrule
\vspace{0mm}
\end{tabular}

Above, \textit{samples} corresponds to the number of nodes for which embeddings are computed, and \textit{order} is the polynomial order of the convolution kernels in $g_\text{conv}$. We use the same order and number of \textit{channels} for both node and graph embedding networks. We use the Adam optimizer with a learning rate of $0.0003$ and run it for at least 600 epochs.

\subsection{Proof of Theorem 2.1}

Let the set of points $X = \{ \bx_1, \ldots, \bx_m\}$ be a $2 \delta$-packing of the domain $D$, such that for any $i\neq j \in [1,m]$ we have $\|\bx_i - \bx_j\|_p > 2 \delta$.
It is known that there exist such set with   
$
    m \geq  \left( \frac{1}{2\delta} \right)^d \frac{\text{vol}(D)}{\text{vol}(B_p(1))}
$
elements.
Further, let $F = \{ \cX = (X, \bc) : \bc \in \mathbb{N}^m \}$ be the family of multisets supported on $X$. 
Since $X$ is a $2\delta$-packing, every $\delta$-injective function $\f$ must yield a different output for every $\cX \in F$, also implying that the domain of $\f$ must be of size at least $|F| = \mathbb{N}^m$.  
The lower bound then follows by noticing that (\textit{i}) according to the pigeonhole principle, the co-domain of any injective function must be at least as large as its domain, and moreover that (\textit{ii}) the set of computable numbers $\mathbb{P}$ is countable (there exists a bijection between $\mathbb{N}$ and $\mathbb{P}$).

\section{Proof of Theorem 2.2} 

Let $P = \{\bp_1, \ldots, \bp_{m}\}$ be a set of $m$ points in $D \subset \bbR^d$ that form a $\delta$-cover of $D$. There exists such a set with $m = O(1/\delta^d)$ points (under mild conditions on $D$).

Further, denote by $C_i  = \{\bx \in D : \|\bx - \bp_i\|_p \leq \|\bx - \bp_j\|_p \; \text{for all}\; j \neq i\} \subset \bbR^d$ the $i$-th cell of the Voronoi tessellation induced by $P$. 

We define the Voronoi-sum aggregation function $\v : \cX \to \mathbb{N}_{\geq 0}^m$ as follows:  
$$[\v(\cX)]_i \defeq \sum_{\bx \in \cX} \indicate{\bx \in C_i},$$
where $\indicate{\cdot}$ is the indicator function.

This function can be seen to be $\delta$-injective:

\begin{lemma}
Each Voronoi-sum aggregation function $\v : \cX \to \mathbb{N}_{\geq 0}^m$ 
is $\delta$-injective, with $\delta = \sup_{i} \sup_{\bx, \by \in C_i} \|\bx - \by\|_p$ the largest diameter of any cell. 
\label{lemma:Voronoi_injective}
\end{lemma}

\begin{proof}
For sake of contradiction, suppose that $\v$ is \textit{not} $\delta$-injective. 
Then there must exist $\cX$ and $\cY$ for which $\v(\cX) = \v(\cY)$, while for \textit{every} bijective map $\phi$ there exists some $\bx \in \cX$ and $\by \in \cY$ such that $\|\phi(\bx) - \by\|_p > \delta$. 

For starters,  suppose that the two multisets are identical except w.r.t. pair $(\bx,\by = \phi(\bx))$. By the construction of the Voronoi diagram, $\bx$ and $\by$ must belong to different cells $C_i$ and $C_j$, respectively. It follows that $[\v(\cX)]_i = [\v(\cY)]_i + 1$ and $[\v(\cX)]_j = [\v(\cY)]_j - 1$.

More generally, order pairs as $(\bx_t,\by_t = \phi(\bx_t))_{t=1}^n$ such that $\|\bx_t - \by_t\|_p > \delta$ for $t\leq k$ and $\|\bx_t - \by_t\|_p \leq \delta$, otherwise. For every $t \in [0,n]$, define the potential vector $\bz_t = \v(\{ \by_\tau\}_{\tau \leq t} ) - \v(\{ \bx_\tau\}_{\tau \leq t})$, with $\bz_0$ being the zero vector.  
Notice that, for every $0<t\leq k$, $\bz_t$ is obtained by moving a unit of potential between two entries of $\bz_{t-1}.$ To model this, we construct a graph with $m$ nodes, each corresponding to an entry of $\bz_t$. For each $t$, we add a directed edge from the node whose value is decreased to that whose value is increased. The relation $\v(\cX) = \v(\cY) \Leftrightarrow \bz_n = 0$ amounts to asserting that the in-degree and out-degree of every node are equal (i.e., the value of a node is increased as many times as it is decreased). The same condition however suffices to guarantee that there exists (at least) one Eulerian circuit, that is, a path starting from any node and returning to the same node by crossing each edge exactly once. W.l.o.g. let the order of the pairs $\cdots, (\bx_{t-1},\by_{t-1}), (\bx_t,\by_t), \cdots$ we have selected correspond exactly to the order of the edges in the Eulerian path.  
Then, there must exist (at least) one sequence of pairwise swaps between the points in consecutive pairs such that every entry of $\bz_t$ is zero---the necessary modification is obtained by iteratively swapping points in consecutive edges along the Eulerian circuit.
Thus, there exists a new mapping function $\phi'$ from $\cX$ to $\cY$ such that every $\bx \in \cX$ belongs to the same cell 
as $\by = \phi'(\bx)$, or equivalently $\max_{\bx \in \cX} \|\phi'(\bx) - \by\|_p \leq 
\delta$, a contradiction.
\end{proof}

To proceed, let $\cX_1, \ldots, \cX_t$ be $t$ multisets for which $\v(\cX_i) \neq \v(\cX_j)$ for all $i\neq j \in [1,t]$. 
Moreover, fix $S$ to be the minimal cardinality subset of $\{1, 2, \cdots, m\}$ such that $[\v(\cX_i)]_{S} \neq [\v(\cX_j)]_{S}$ for all $i\neq j \in [1,t]$. That is, we can still distinguish the outputs of $\v$ if we only know the entries in set $S$. The cardinality of $S$ is at most $t$ (since $t$ dimensions suffice to distinguish between any $t$ non-identical vectors).

If we fix $\textsf{sum}(\cX) = [\v(\cX)]_{S}$, we can always write
$$
\textsf{sum}(\cX_i) = \sum_{\bx \in \cX_i} \varphi(\bx) = [\v(\cX_i)]_{S} \neq [\v(\cX_j)]_{S} = \sum_{\bx \in \cX_j} \varphi(\bx) = \textsf{sum}(\cX_j) \quad \text{for all} \quad i\neq j \in [1,t].
$$
Specifically, $\varphi(\bx)$ has one entry for each element $l \in S$ and $[\varphi(\bx)]_l = \indicate{\bx \in C_{l}}$. The claim then follows since the preceding construction holds for every $\cX_1, \ldots, \cX_t$.

\section{Proof of Theorem 2.3}

To prove our claim, we will eventually need to identify functions $\varphi$ such that that $\p(\cX_i) \neq \p(\cX_j)$ for all $i\neq j$. 
But before doing so, we need to establish some preliminary constructions.

We construct a $\delta$-injective Voronoi-sum function $\v$ based on the point set $P^d = P \times P \times \cdots \times P = \{\bp_1, \cdots, \bp_{m_v}\}$ corresponding to a $d$-ary Cartesian power of $P$ and $m_v = b^d$. 
Under the $\ell_\infty$ norm, the cell $\{ \bx \in \bbR^d: \| \bx - \bp \|_\infty < \| \bx - \bp' \|_\infty ~ \forall \bp' \neq \bp\} $ associated with each point $\bp \in P^d$ is a hypercube $B_{\infty}(\bp,w) = \{ \bx \in \bbR^d : \| \bp - \bx\|_\infty \leq w\}$ of radius $w = \delta/2$ centered at $\bp \in \bbR^d$.

Now, we set $\by_i = \varphi(\bx_i) \in \bbR^d$ and, accordingly, define $\cY_i= \{ \varphi(\bx) : \bx \in \cX_i \}$. With this in place we can express the entries of $\p(\cX)$ in relation to those of $\h(\cY)$:
\begin{align*}
    [\p(\cX)]_{i,l}
    = \sum_{\bx \in \cX} \kappa(| [\varphi(\bx)]_i - p_l|)
    &= \sum_{\by \in \cY} \kappa(| [\by]_i - p_l|) \\
    &\overset{(A)}{=}  \sum_{\by \in \cY} \indicate{[\by]_i \in B_\infty(p_l,w) } \\
    &=  \sum_{\by \in \cY} \sum_{ [\bp_j]_i = p_l } \indicate{\by \in B_\infty(\bp_j,w) } \\    
    &=  \sum_{ [\bp_j]_i = p_l } \sum_{\by \in \cY} \indicate{\by \in B_\infty(\bp_j,w)  } 
    = \sum_{ [\bp_j]_i = p_l } [\h(\cY)]_{j}
\end{align*}
Above, equality $(A)$ holds because $\kappa$ is selected to be uniform (and equal to one) within $B_\infty(p_l,w)$.

Thus, there exists a matrix $\bA$ of size $m \times m_v $ (equiv. $ bd \times b^d$), such that 
$$
    \vec{\p(\cX)} = \bA \, \h(\cY).
$$
The behavior of this matrix is visualized in Figure~\ref{fig:hypercube}.

\begin{figure}[t]
    \centering
    \includegraphics[width=1\columnwidth]{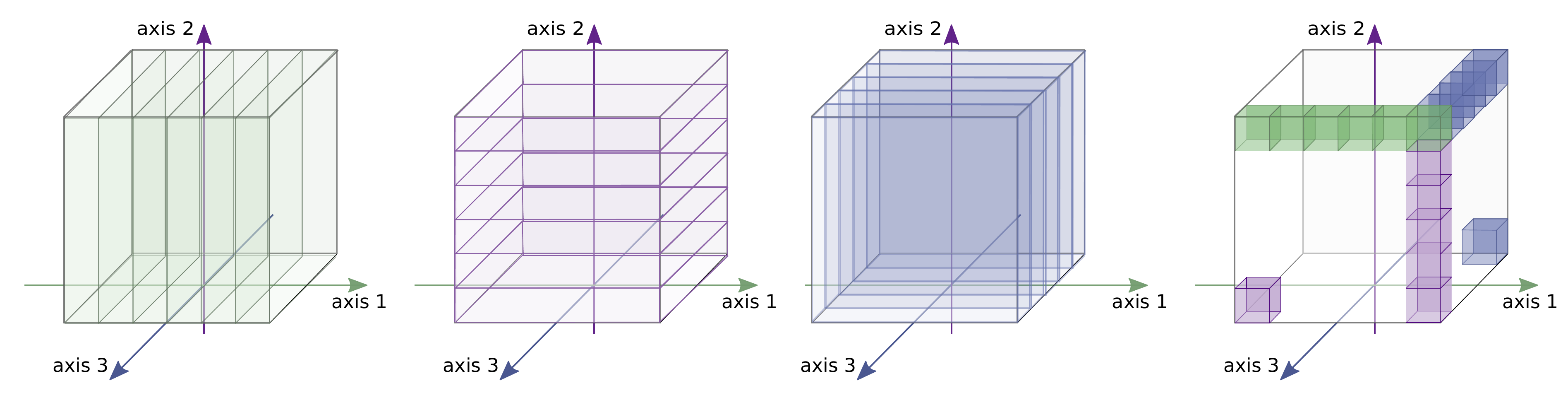}
    \caption{The first three figures depict the behavior of matrix $\bA$ for $k=3$ and $b = 6$. Each rectangle corresponds to a row of $\bA$ and sums the elements of $[\h(\cY)]_j$ for which $[\bp_j]_i = p_l$ with $l \in [1,b]$ and $i\in [1,d]$. For convenience, we distinguish the three dimensions by different colors (green, purple, blue). The rightmost figure shows the index set constructed in the proof of Lemma~\ref{lemma:support}}
    \label{fig:hypercube}
\end{figure}

We next derive sufficient conditions such that a vector is not in the nullspace of $\bA$.
\begin{lemma}
For any $S \subseteq \{1, \ldots, m_v\}$, let $D_S = \{ \bh \in \bbR^{m_v} : [\bh]_{i} = 0 ~~\text{for all} ~~  i \notin S\} $ be the set of all vectors whose support is a subset of $S$. There exists $S^*$ with $|S^*| = m$, such that $D_{S^*} \perp \nullspace{\bA}$.
\label{lemma:support}
\end{lemma}

\begin{proof}
Consider any vector $\bh \in D_{S^*}$. Since each entry of $[\bh]_j$ corresponds to a point $\bp_j \in P^d$, we talk about points and entries interchangeably---e.g., we will define $S^*$ as a subset of $P^d$.  

Specifically, we define $S^*$ by the following recurrence relation:
$$
    S_1 = \{ \bp \in P^d : [\bp]_j = 1, ~ \forall j \neq 1  \} 
$$
and 
$$
    S_i = S_{i-1} \cup \{ \bp \in P^d : [\bp]_j = 1, ~ \forall j \neq i \} \cup \{ \bp \in P^d : [\bp]_i = -1, [\bp]_{i-1} = -1,  [\bp]_j = 1, ~ \forall j \neq {i-1},i \},
$$
for $i = 2, \ldots, d$. 
Since exactly $b$ points are added for every iteration $i$, $S^* = S_k$ contains $db = m$ points.

The construction can be seen in Figure~\ref{fig:hypercube} (right) for $d=3$. Specifically, the green cubes correspond to the points in $S_1$, the purple to the points $S_2 \setminus S_1$, and finally the blue to the points $S_3 \setminus S_2$.

Now, consider the $j$-th entry $[\bh]_j$ of any $\bh \in D_{S^*}$. We distinguish two cases: 
(\textit{interior}) If $\bp_j$ lies in the interior of the hypercube (it has no coordinate equal to $\pm$1), then by construction there exists a row of $\bA$ with support intersects with $S^*$ exactly at $[\bh]_j$. (The corresponding equation corresponds to a sum over all entries with $i$-th coordinate equal to $[\bp_j]_i$). In other words, $[\bh]_j$ can be uniquely determined from $\bA \bh$. 
(\textit{boundary}) We move on to points in the boundary of the hypercube. There are $2k$ such points (two introduced by each iteration of the recurrence relation). W.l.o.g. suppose that all entries in the interior of the hypercube are zero (otherwise, non-zero interior entries can be eliminated by the aforementioned argument). Then, in the rows of $\bA$ one may find $2k$ linear equations involving solely boundary points. Thus, also in this case $[\bh]_j$ can be uniquely determined from $\bA \bh$, which concludes the proof.
\end{proof}

Moreover, let $S$ be the minimal cardinality subset of $\{1, 2, \cdots, m_v\}$ such that $[\h(\cX_i)]_{S} \neq [\h(\cX_j)]_{S}$. That is, we can still distinguish the vectors if we only know their entries in set $S$. The cardinality of this set is lower bounded by $|S| \geq t$ (since $t$ dimensions suffice to distinguish between $t$ vectors). 
Since $\h$ tessellates the input domain into $m_v$ cells $B_j \defeq B_{\infty}(\bp_j,w)$, each associated with a different point $\bp_j \in P^d$, restricting our attention to the entries in $S$ amounts to discarding from the multisets every element that does not fall into $B_j$ with $j \in S$. 
 
On the other hand, from Lemma~\ref{lemma:support}, we know that there exists a set $S^*\subset \{1, 2, \cdots, m_v\}$ of cardinality $m$ such that any vector in $\bh \in D_{S^*} \notin \nullspace{\bA}$. 
As long as $|S| \leq |S^*|$ (or equiv. $t \leq m$), we can always select $\varphi$ to be an injective map from $\{ B_j : j \in S\}$ to $\{B_{j'} : j' \in S^*\}$. This means that, for any $j \in S$ and $\bx \in B_j$, there exists a single $j' \in S^*$ such that $\varphi(\bx) \in B_{j'}$. We notice that there are exactly $m!=m (m-1) \cdots 1$ different ways to map $\{ B_j : j \in S\}$ onto $\{B_{j'} : j' \in S^*\}$ injectively.

Under this transformation, if $[\h(\cX_i)]_{S} \neq [\h(\cX_j)]_{S}$ then also $[\h(\cY_i)]_{S^*} \neq [\h(\cY_j)]_{S^*}$. Finally, assuming w.l.o.g. that $[\h(\cY_j)]_{j} = 0$ for all $j \notin S^*$ (we can always select $\varphi$ to achieve this), the aforementioned condition suffices to guarantee that also 
$$\vec{\p(\cX_i)} = \bA \h(\cY_i) \neq \bA \h(\cY_j) = \vec{\p(\cX_i)},$$
as desired.
\end{appendices}

\end{document}